\documentclass[letterpaper]{article}
\usepackage{uai2019}
\usepackage[margin=1in]{geometry}
\usepackage{microtype}
\usepackage{graphicx}
\usepackage{subfigure}
\usepackage{tabularx} 
\usepackage{booktabs} 
\usepackage{amsmath, amsthm} 
\usepackage{amssymb} 
\usepackage{xcolor}
\usepackage{dblfloatfix} 
\usepackage[square,sort, round]{natbib}
\usepackage{lmodern}
\usepackage{needspace}

\usepackage{times}

\usepackage{pifont}
\usepackage{tikz} 
\usetikzlibrary{automata, positioning, shapes.geometric}

\usepackage{algorithm} 
\usepackage{algorithmic}

\usepackage{hyperref}

\newtheorem{theorem}{Theorem}[section] 
\newtheorem{lemma}{Lemma}[section] 
\newtheorem{assumption}{Assumption}[section]
\newtheorem{notation}{Notation}[section]
\newtheorem{remark}{Remark}[section]

\title{A Bayesian Approach to Robust Reinforcement Learning}

\author{
  Esther Derman \\
  Technion, Israel\\
  \texttt{estherderman@campus.technion.ac.il} 
  \And
  Daniel Mankowitz \\
  Deepmind, UK\\
  \texttt{dmankowitz@google.com} 
  \AND
  Timothy Mann \\
  Deepmind, UK \\
  \texttt{timothymann@google.com} 
  \And
  Shie Mannor \\
  Technion, Israel\\
  \texttt{shie@ee.technion.ac.il} \\
} 

\begin{document}
\maketitle

\begin{abstract}
Robust Markov Decision Processes (RMDPs) intend to ensure robustness with respect to changing or adversarial
system behavior. In this framework, transitions are modeled as arbitrary
elements of a known and properly structured \textit{uncertainty set} and a robust optimal policy can be derived under the worst-case scenario.
In this study, we address the issue of learning in RMDPs using a Bayesian approach.  
We introduce the Uncertainty Robust Bellman Equation (URBE) which encourages safe exploration for adapting the uncertainty set to new observations while preserving robustness. We propose a URBE-based algorithm, DQN-URBE, that scales this method to higher dimensional domains.
Our experiments show that the derived URBE-based strategy leads to a better trade-off
between less conservative solutions and robustness in the presence of model misspecification. In addition, we show that the DQN-URBE algorithm can adapt significantly faster to changing dynamics online compared to existing robust techniques with fixed uncertainty sets.
\end{abstract}

\section{INTRODUCTION} 
\label{introduction}

Markov Decisions Processes (MDPs) are used for solving sequential decision making problems with varying degrees of uncertainty. Two types of uncertainty are encountered: the internal uncertainty due to the stochasticity of the system and the uncertainty in the transition and reward parameters \citep{biasMannor}. In order to mitigate the second type of uncertainty, the Robust-MDP (RMDP) framework considers the unknown parameters to be a member of a known uncertainty set \citep{nilim, iyengar, wieseman}. An optimal solution to the robust RL problem then corresponds to the strategy that maximizes the worst-case performance and it can be derived using dynamic programming \citep{iyengar, scalingRMDP}. 

However, planning in RMDPs often leads to overly conservative solutions. There are two reasons for this: Firstly, the uncertainty set has to be \textit{rectangular} in order for the problem to be computationally tractable, which means that it must be structured as sets of MDP models that are independent for each state \citep{wieseman}. 
Let give an intuition of the reason why the resulting policy may be too conservative. Suppose a chess player wants to protect himself against an adversary he has some prior on, which can be modeled as a set of transition models. Suppose also that the agent optimizes its next move according to the worst-case scenario from the set of models. Then, the rectangularity assumption implies that the agent considers the worst-case transition for each game configuration independently, although there are several chances that some configurations are incompatible during the same round of game. It comes out that this robust strategy is overly conservative. 
Attempts to circumvent rectangular uncertainty sets in RMDPs include the works \citep{LDST, krectangularity, petrikNonrectangular} and more recently, \citep{goyalRectangularity}. 

Secondly, the difficulty of constructing uncertainty sets can result in too large sets and consequently lead to overly-pessimistic strategies \citep{petrikBayes2}. Proposals for \textbf{learning} an uncertainty set in a data-driven manner have rarely been addressed in RL literature. In \citep{petrikBayes2}, the authors designed a robustification procedure that builds \textit{safe} uncertainty sets upon optimal value functions. Their Bayesian method starts from a posterior distribution on transitions and constructs possibly discontinuous sets by iteratively solving optimization problems. Although it leads to tighter uncertainty sets, their algorithm proceeds offline with a fixed batch of data and is not scalable due to its high computational complexity. Moreover, their technique assumes the data to be generated by one fixed but unknown model, so it has not been tested against changing dynamics. 

Previous work \citep{OLRMfull} has used the ``optimism in face of uncertainty" (OFU) principle \citep{UCRL2} to detect adversarial state-action pairs online and compute an optimistic minimax policy accordingly. Although these methods have been proven to be statistically efficient, they require an exhaustive computation for each state-action pair. This leads to solutions that are intractable for all but small problems. Also, as is common in optimistic approaches, the resulting performance is highly influenced by the underlying analysis \citep{PSRL}.

Based on a Bayesian approach, Thompson sampling \citep{TS}  has been shown to be more efficient than OFU methods in RL problems. Previous work \citep{whyPS} has stressed the advantages of posterior sampling methods over existing algorithms driven by optimism. However, most of the existing work on posterior sampling methods studied finite tabular MDPs. The Uncertainty Bellman Equation (UBE) work \citep{UBE} addressed this shortcoming and proposed an online algorithm that scales naturally to large domains. Their method learns the posterior variance of the value for guiding exploration when the true dynamics of the MDP are unknown. Yet, their approach does not deal with adversarial transitions. 
While better exploration can potentially lead to more efficient learning and improved solutions, it cannot protect itself against sudden (potentially) adversarial changes in the underlying dynamics of the environment. This is especially true when the domain is large and/or hard to explore efficiently.

In this work, we introduce a Bayesian framework for robust RL and address the first Bayesian algorithm that (1) accounts for changing dynamics online and (2) tackles conservativeness thanks to a variance bonus that detects changes in the level of adversity. This variance is proven to satisfy an Uncertainty Robust Bellman Equation (URBE), that is estimated using dynamic programming. Besides being scalable to complex domains, our approach leads to less conservative results than existing planning methods for RMDPs while ensuring robustness to model misspecification \footnote{In this work, model misspecification will designate any perturbation of the system's dynamics.}.
Our experiments illustrate an improved trade-off between overly conservative, robust behaviour and less conservative, improved performance for the resulting DQN-URBE policy.

\textbf{Main Contributions:} 
To summarize, our specific contributions are: (1) The URBE which encourages safe exploration and prevents overly conservative solutions; (2) The DQN-URBE which scales and utilizes URBE to learn less conservative solutions that are still robust to model misspecification; (3) Adaptability of DQN-URBE to changing dynamics online.



\section{BACKGROUND}

\textbf{Bayesian Reinforcement Learning} Bayesian RL leverages methods from Bayesian inference to incorporate prior information about the Markov model into the learning process. Model-based Bayesian RL \citep{deardenMB, strens,PSRL} express prior information on parameters of the Markov process instead. In the work \citep{deardenMF}, the authors introduced Bayesian Q-learning to learn the posterior distribution of the Q-values in the model-free setting. One major advantage of Bayesian RL is that it can benefit from prior information on the problem to tackle the exploration-exploitation dilemma (see \citep{BRLS} for a full review).

\textbf{Robust MDP} A robust MDP is a tuple $\langle \mathcal{S}, \mathcal{A}, r, \mathcal{P} \rangle$ where $\mathcal{S}$ and $\mathcal{A}$ respectively denote state and action spaces. The mapping $r :\mathcal{S}\times \mathcal{A}\rightarrow \mathbb{R}$ defines the immediate bounded reward function and $\mathcal{P}$ is a set of transition matrices that models the ambiguity in the transition distributions. As common in the robust RL literature, we assume
$\mathcal{P}$ to be structured as a cartesian product $\bigotimes _{s\in \mathcal{S}, a\in\mathcal{A}} \mathcal{P}_{s,a}$, which is also known as the \textit{state-action rectangularity} assumption \citep{wieseman}. In RMDPs, this implies that the nature can choose the worst-transition independently for each state and action.

\textbf{Robust DQN} The DQN algorithm uses a neural network as a function approximation of the Q-value and learns its parameters by optimizing a TD-loss. Similarly, the \textbf{robust} Bellman equation uses a robust TD-error as a loss criterion for learning a minimax policy \citep{modelmismatch}. In \citep{kalman}, a robust counterpart to DQN called RTD-DQN has been introduced and shown to be robust to model misspecification. 

\section{ROBUST MARKOV DECISION PROCESSES}
We consider a robust MDP $\langle \mathcal{S}, \mathcal{A}, r, \mathcal{P} \rangle$ of finite state and action spaces, where each episode has finite horizon length $H\in\mathbb{N}$. At step $h$, an agent is in state $s^h$, selects an action $a^h$ according to a stochastic policy $\pi^h: \mathcal{S}  \rightarrow \Delta_\mathcal{A}$ that maps each state to a probability distribution over the action space, $\Delta_\mathcal{A}$ denoting the set of distributions over $\mathcal{A}$. Thus, for all steps $h = 1,\cdots, H$, $\sum_{a\in\mathcal{A}}\pi_{sa}^h = 1$. After choosing an action, the agent gets a deterministic reward $r^h$ bounded by $R_{\max}$, and transitions to state $s^{h+1}$ according to an arbitrary transition $p_{s^h,a^h}\in\mathcal{P}_{s^h, a^h}$ which we will rather write as $p_{sa}^h\in \mathcal{P}_{sa}^h$ with a slight abuse of notation.

The robust action-value (or robust Q-value) at step $h$, state $s$, action $a$ and under policy $\pi:= (\pi^1, \cdots,\pi^H)$ is the expected discounted return under the worst-case scenario resulting from taking action $a$ at $s$ and following policy $\pi$ thereafter:
\begin{equation*}
    Q_{sa}^h := \inf_{p\in\mathcal{P}}\mathbb{E}\left[ \sum_{l = h}^H \gamma^{l-h}r^l \mid s^h =s, a^h = a, \pi,p\right],
\end{equation*}
with $\gamma \in (0,1]$. Likewise, the robust value at state $s$ under policy $\pi$ is $V^h(s) := \mathbb{E}_{a\sim \pi_{s}^h}[Q_{sa}^h]$. When it is clear from the context, we suppress the dependence on $\pi$ for notational convenience. 

A robust optimal policy is derived by maximizing the expected worst-case discounted return:
\begin{equation*}
    J(\pi) := \inf_{p\in\mathcal{P}} \mathbb{E}^{\pi, p}\left[\sum_{h = 1}^{H}\gamma^{h-1} r^h\right] = V^1(s).
\end{equation*}

Assuming a rectangular structure on $\mathcal{P}$, the robust Bellman operator $\mathcal{T}^h$ for policy $\pi$ at step $h$ relates the robust value at $h$ to the robust value at following steps \citep{nilim}:
\begin{equation*}
\label{rbe}
   \mathcal{T}^h Q_{sa}^{h+1} = r_{sa}^h + \gamma \inf_{p\in\mathcal{P}}\sum_{s'\in\mathcal{S}, a'\in\mathcal{A}}\pi_{s'a'}^hp_{sas'}^hQ_{s'a'}^{h+1},
\end{equation*}
where $r_{sa}^h$ denotes the immediate reward at step $h$ while being in state $s$ and executing action $a$.
Computing the second term in the robust Bellman operator amounts to solving a robust optimization problem where the robust constraints are given by $\mathcal{P}$. In fact, the main challenge of robust optimization is to build an uncertainty set such that the solution stays tractable without being overly conservative.

\section{THE UNCERTAINTY ROBUST BELLMAN EQUATION}
In this section, we introduce our Bayesian framework for robust RL where we have a prior over the transition model. Our approach is based on the following procedures: (a) building posterior uncertainty sets, (b) approximating posterior distribution over robust Q-values. Next, we introduce an upper bound on the variance of the posterior over robust Q-values and show that it satisfies a Bellman recursion, which we call the Uncertainty Robust Bellman Equation (URBE). Proofs are deferred to the Appendix.

\subsection{POSTERIOR UNCERTAINTY SETS}
Define $\varphi_p$ as a prior distribution according to which state transitions are generated. Assume furthermore that $\varphi_p$ is a product of $\vert \mathcal{S}\vert\cdot\vert \mathcal{A}\vert$ independent Dirichlet priors on each distribution $p_{sa}$ over next states, that is $\varphi_p = \prod_{s,a}\varphi_{sa}$, where $\varphi_{sa}$ is Dirichlet. Given an observation history $\mathcal{H} = \langle (s_1, a_1), (s_2, a_2), \dots, (s_h, a_h) \rangle \in (\mathcal{S} \times \mathcal{A})^h$ induced by a policy $\pi$ and a confidence level $\psi_{sa}\in \mathbb{R}^+$ for each state-action pair, we can construct a subset of transition probabilities $\Delta_{\mathcal{S}}$:
\begin{equation*}
\begin{split}
 \widehat{\mathcal{P}}^h_{sa}(\psi_{sa}) = 
 \left\{p_{sa}\in\Delta_{\mathcal{S}} : \lVert p_{sa} - \bar{p}_{sa} \rVert_{1} \leq \psi_{sa}  \right\} 
 \end{split}
\end{equation*}
where $\bar{p}_{sa}$ is the nominal transition given by $\bar{p}_{sa} = \mathbb{E}[p_{sa}\mid \mathcal{H}]$. If $\mathcal{H}$ is fixed, this construction falls into the definition of a Bayesian confidence interval introduced in \citep{petrikBayes2}.

Such a construction forms a rectangular uncertainty set $\widehat{\mathcal{P}}^h(\psi) := \bigotimes_{s,a}\widehat{\mathcal{P}}^h_{sa}(\psi_{sa})$ 
We call it a \textit{posterior uncertainty set} and will omit the dependence in $\psi$ for ease of notation. 
To derive smaller posterior regions of posterior confidence level $\alpha$, one can proceed as described in 
\citep{petrikBayes2, guptaDRO} by minimizing $\mathbb{P}(\lVert p_{sa} - \bar{p}_{sa} \rVert_{1} > \psi_{sa} \mid \mathcal{H}) < \frac{\alpha}{\vert \mathcal{S}\vert \vert \mathcal{A}\vert}$ with respect to $\psi_{sa}$ that satisfies the constraint. However, in our case, the data set is not fixed so the nominal transition changes, which raises tractability issues. Therefore, $\psi_{sa}$ is remained fixed without being optimized.

\subsection{POSTERIOR OVER ROBUST Q-VALUES}
The simulation proceeds as follows: at each episode $t$, we sample a transition matrix according to $\varphi_{p}$. For a fixed policy $\pi$, we collect observation history and update the posterior distribution accordingly. We then construct a posterior uncertainty set $\widehat{\mathcal{P}}^h_{sa}(\psi_{sa})$ based on all observed data from all previous episodes. A posterior over robust Q-values can then be obtained via the following equation:
\begin{equation*}
\begin{split}
   \widehat{Q}_{sa}^h &= r_{sa}^h + \gamma \inf_{p \in \widehat{\mathcal{P}}_{sa}^h}\sum_{s', a'}\pi_{s'a'}^h p_{sas'} \widehat{Q}_{s'a'}^{h+1} \enspace ,
   \end{split}
\end{equation*}
with $\widehat{Q}_{sa}^{H+1} = 0$. The quantity $\widehat{Q}_{sa}^h$ is a random variable whose variability comes from the nominal transition used in constructing posterior uncertainty sets, from the stochasticity of the policy and the dynamics of the sampled MDP. We further define
\begin{equation}
\label{p-post}
\begin{split}
   \widehat{p}_{sa}^h &\in \arg \min_{p \in \widehat{\mathcal{P}}_{sa}^h}\sum_{s', a'}\pi_{s'a'}^h p_{sas'} \widehat{Q}_{s'a'}^{h+1}
\end{split}
\end{equation}
as a worst-case transition at step $h$.


\subsection{POSTERIOR VARIANCE OF ROBUST Q-VALUES}
\label{posterior-variance-section}
For the regular MDP setting, the work \citep{UBE} showed that the conditional variance of posterior Q-values can be bounded by a quantity that satisfies a Bellman recursion formula. In Bayesian robust RL, a similar upper bound by a robust Bellman update can be derived. The key difference with \citep{UBE} is that they evaluate posterior Q-values according to one transition model whereas we evaluate robust Q-values according to a posterior uncertainty set. 

Let first introduce some notation:

\begin{notation}
Define $\mathcal{F}_t$ as a minimal sigma-algebra that contains all of the available information up to episode $t$ (\textit{e.g.} all observed states, actions and rewards). Denote by $\mathbb{E}_t[X]$ the expectation of random variable $X$ conditioned on $\mathcal{F}_t$. Similarly, the conditional variance $\mathbf{var}_t(X)$ is defined as: $\mathbf{var}_t X := \mathbb{E}_t\left[\left(X-\mathbb{E}_t[X]\right)^2\right]$.
\end{notation}

As common in literature \citep{randomizedValueFunctions, UBE}, we make the following assumptions:
\begin{assumption}
\label{acyclic}
For any episode, the graph resulting from a worst-case transition model is directed and acyclic.
\end{assumption}

\begin{assumption}
\label{Rbound}
For all $(s,a)\in\mathcal{S}\times\mathcal{A}$, the rewards are bounded: $-R_{\max} \leq r_{sa}\leq R_{\max}$. This implies that the robust Q-value is bounded as well: $\mid Q_{sa}^h\mid \leq H R_{\max} =: Q_{\max}$.
\end{assumption}

These assumptions enable to state the following result.

\begin{lemma}
\label{variance}
Under Assumptions \ref{acyclic} and \ref{Rbound}, for any worst-case transition $\widehat{p}$ as defined in equation (\ref{p-post}), the conditional variance of the robust Q-values under the posterior distribution satisfies the robust Bellman inequality:
\begin{equation*}
    \mathbf{var}_t \widehat{Q}_{sa}^h \leq \nu_{sa}^h + \gamma^2\sum_{s', a'}\pi_{s'a'}^h \mathbb{E}_t \left(\widehat{p}_{sas'}^h \right) \mathbf{var}_t\widehat{Q}_{s'a'}^{h+1},
\end{equation*}
with $\mathbf{var}_t\widehat{Q}^{H+1} = 0$ and $\nu_{sa}^h := Q^2_{\max} \sum_{s'\in\mathcal{S}}\frac{\mathbf{var}_t \widehat{p}_{sas'} ^h}{ \mathbb{E}_t \widehat{p}_{sas'} ^h}$.
\end{lemma}

This lemma enables us to establish the Uncertainty Robust Bellman Equation (URBE). 

\begin{theorem}[Solution of URBE]
\label{urbe-thm}
For any worst-case transition $\widehat{p}$ as defined in equation (\ref{p-post}) and any policy $\pi$, under Assumptions \ref{acyclic} and \ref{Rbound}, there exists a unique mapping $w$ that satisfies the uncertainty robust Bellman equation:
\begin{equation}
\label{urbe}
    w_{sa}^h = \nu_{sa}^h +\gamma^2 \sum_{s'\in\mathcal{S}, a'\in\mathcal{A}} \pi_{s'a'}^h \mathbb{E}_t(\widehat{p}_{sas'}^h)w_{s'a'}^{h+1}, 
\end{equation}
for all $(s,a)\in \mathcal{S}\times\mathcal{A}$ and $h=1, \cdots, H$ where $w^{H+1} = 0$. Furthermore, $w\geq \mathbf{var}_t \widehat{Q}$. 
\end{theorem}

A classical difficulty in Bayesian approaches is to compute the posterior distribution. The Bayesian central limit theorem (Result 8 in \citep{Berger}) ensures that under smoothness assumptions on the prior and likelihood functions, the posterior distribution converges to a Gaussian distribution. Thus, we get around tractability issues by approximating the posterior over robust Q-values as $\mathcal{N}(\bar{Q}, \mathbf{diag}(w))$, where $w$ is the solution to URBE and $\bar{Q}$ is the unique solution to:
\begin{equation*}
\begin{split}
   \bar{Q}_{sa}^h &= r_{sa}^h + \gamma \sum_{s', a'}\pi_{s'a'}^h \mathbb{E}_t(\widehat{p}_{sas'}^h) \bar{Q}_{s'a'}^{h+1},
   \end{split}
\end{equation*}
for $h = 1,\cdots, H$, and $\bar{Q}^{H+1}=0$.

\begin{remark}
We should emphasize that the quantity $\mathbb{E}_t(\widehat{p}_{sas'}^h)$ is the conditional expectation of the worst-case transition, and depends on the robust Q-values. Therefore, it is different from the nominal transition which only depends on observations.
\end{remark}


\section{ESTIMATION OF THE ROBUST LOCAL UNCERTAINTY}
\label{robust-local-uncertainty}
Lemma \ref{variance} reveals a quantity $\nu$ that only depends on local state and action pairs. We call it the \textit{robust local uncertainty}, since it also depends on the worst-case transitions. In this section, we present a practical method for estimating this quantity, which will be useful for implementing learning algorithms that take advantage of Theorem \ref{urbe-thm}. We first present the tabular representation in the robust setup. We then recall the linear function representation and neural network architectures from \citep{UBE} which directly enable to scale up the robust local uncertainty estimate. 

\subsection{TABULAR CASE}
Assume a Dirichlet prior $\varphi_{sa} := (\varphi_{sas'})_{s'\in \mathcal{S}}$ on  transitions that depart from $(s,a)\in \mathcal{S}\times \mathcal{A}$. For all $h = 1,\cdots, H$, the posterior distribution is Dirichlet:
\begin{equation*}
p_{sa}^h \mid \mathcal{F}_t \sim \text{Dir}(\varphi_{sa} + n_{sa}^h)   
\end{equation*}
where $n_{sa}^h := (n_{sas'}^h)_{s'\in \mathcal{S}}$ is the vector of counts for observation $(s,a,s')$ at step $h$, up to episode $t$. Therefore, given a posterior uncertainty set $\widehat{\mathcal{P}}^h_{sa}$ at step $h$, any $p_{sa}^h\in \widehat{\mathcal{P}}^h_{sa}$ satisfies the following:
\begin{equation*}
  \mathbf{var}_t p_{sas'} ^h \leq \frac{\varphi_{sas'} + n_{sas'}^h}{\left(\sum_{s'\in\mathcal{S}}(\varphi_{sas'} + n_{sas'}^h)\right)^2},
\end{equation*}
\begin{equation*}
  \mathbb{E}_t p_{sas'} ^h = \frac{\varphi_{sas'} + n_{sas'}^h}{\sum_{s'\in\mathcal{S}}(\varphi_{sas'} + n_{sas'}^h)}.
\end{equation*}
Since $\widehat{\mathcal{P}}^h_{sa}$ is a closed set, $\widehat{p}_{sas'}$ also satisfies these inequalities and
\begin{equation*}
  \frac{\mathbf{var}_t \widehat{p}_{sas'} ^h}{\mathbb{E}_t \widehat{p}_{sas'} ^h} \leq \frac{1}{\sum_{s'\in\mathcal{S}}(\varphi_{sas'} + n_{sas'}^h)} \leq \frac{1}{n_{sa}^h},
\end{equation*}
where $n_{sa}^h$ is the visit count of the agent from state $s$ and action $a$.
It follows that $\nu_{sa}^h \leq Q^2_{\max} \mid S\mid /n_{sa}^h$. Therefore, similarly to the non-robust setup, the robust local uncertainty can be modeled as a positive constant $\beta^2$ divided by the visit count $n_{sa}^h$.


\subsection{FUNCTION APPROXIMATION}
We now adapt the robust local uncertainty estimate to function approximation representations. Let $\widehat{Q}_{sa}^h \approx \phi_s^T\theta_a$ be a linear function approximation for the robust Q-values, where $\phi : \mathcal{S}\rightarrow \mathbb{R}^d$ designates state features and $\theta_a$ are parameters learned for each action $a \in \mathcal{A}$. Using the inverse count estimator $(\widehat{n}_{sa}^h)^{-1} = \phi_s^T (\Phi_a^T \Phi_a)^{-1} \phi_s$ introduced in the work \citep{UBE}, where $\Phi_a$ is the matrix of $\phi_s$-s stacked row-wise with action $a$ being taken at $s$, we estimate the robust local uncertainty by $\widehat{\nu}_{sa}^h = \beta^2 \phi_s^T (\Phi_a^T \Phi_a)^{-1} \phi_s$. As it receives a new sample $\phi$, the agent needs to update the matrix $\Sigma_a := (\Phi_a^T \Phi_a)^{-1}$, which can be implemented efficiently via the Sherman-Morrison-Woodbury formula \citep{golub}:
\begin{equation}
\label{sigmaUpdate}
    \Sigma_a^+ := \Sigma_a - (\Sigma_a \phi \phi^T \Sigma_a)/(1 + \phi^T\Sigma_a\phi).
\end{equation}
The neural network representation proceeds similarly, provided that we treat all layers as feature extractors and apply a linear activation function to the last layer. In that case, we still have  $\widehat{Q}_{sa}^h \approx \phi_s^T\theta_a$, where $\phi_s$ is the output of the last network layer for state $s$ and $\theta_a$ are the parameters of the last layer for action $a$. We use this technique in Algorithm \ref{dqn-urbe}.

\section{URBE-BASED ALGORITHMS}
\subsection{URBE ALGORITHM}
The URBE algorithm is described in Algorithm \ref{urbe-algo}. Its structure is similar to \citep{PSRL}, but involves using robust dynamic programming so as to learn a robust policy as well as posterior variance of the robust Q-values. At the beginning of each episode, an MDP model is sampled according to the current posterior distribution. The posterior uncertainty set is also updated according to new observations. Robust Q-values and its posterior variance are then computed using dynamic programming. At each step, the agent acts greedily with respect to the robust Q-values plus the posterior variance. 

\begin{algorithm}[!h]
   \caption{URBE}
   \label{urbe-algo}
\begin{algorithmic}
   \STATE {\bfseries Input:} Prior distribution $\phi_p$, confidence level $\psi$, $t=1$
    \STATE{\bfseries Initialize:} $t=1$. State and action $(s, a)\in\mathcal{S}\times\mathcal{A}$.
   \FOR {episodes $t = 1,\cdots$}
   \STATE Sample MDP $\sim \phi_p$
   \STATE Observe $s'$ and receive reward $r$\\
   \STATE Update posterior $\varphi_p$ and posterior uncertainty set $\widehat{\mathcal{P}}^h$
   \STATE Compute $\widehat{Q}_{s'b}^h$ and $w_{s'b}^h$ for all action $b$\\
   \STATE Sample $\zeta_b \sim\mathcal{N}(0,1)$ for all $b$ and compute:
   \begin{equation*}
       a' = \arg \max_b \left(\widehat{Q}_{s'b}^h + \zeta_b \sqrt{w_{s'b}^h} \right)
   \end{equation*}
\STATE Take action $a'$
\STATE $s\leftarrow s', a\leftarrow a'$
   \ENDFOR
\end{algorithmic}
\end{algorithm}

\begin{algorithm}[!h]
   \caption{DQN - URBE}
   \label{dqn-urbe}
\begin{algorithmic}
   \STATE {\bfseries Input:} Neural network for robust $Q$ and $w$ estimates;
   Robust DQN subroutine $\mathtt{robust DQN}$; Hyperparameter $\beta>0$
    \STATE{\bfseries Initialize:} $\Sigma_a = \mu \cdot I$ for $a\in\mathcal{A}$ with $\mu >0$; Initial state and action $(s, a) \in \mathcal{S}\times\mathcal{A}$\\

   \FOR {$t = 1,\cdots$}
   \FOR{$h=2$ {\bfseries to} $H+1$}
   \STATE Retrieve $\phi(s)$ from robust $Q$-network
   \STATE Observe $s'$ and receive reward $r$\\
   \STATE Compute $\widehat{Q}_{s'b}^h$ and $w_{s'b}^h$ for all action $b$\\
   \STATE Sample $\zeta_b \sim\mathcal{N}(0,1)$ for all $b$ and compute:
   \begin{equation*}
   \begin{split}
       a' &= \arg \max_b \left(\widehat{Q}_{s'b}^h + \beta\zeta_b \sqrt{w_{s'b}^h} \right) \text{ and }\\
y &= \left\{
    \begin{array}{ll}
        \phi(s)^T\Sigma_a\phi(s) & \mbox{if } h = H+1 \\
        \phi(s)^T\Sigma_a\phi(s)  + \gamma^2 w_{s'a'}^h& \mbox{otherwise}
    \end{array}
\right.
\end{split}
\end{equation*}
\STATE Take gradient step on $w$ w.r.t. loss $(y - w_{sa}^{h-1})^2$
\STATE Update robust Q-values using $\mathtt{robust DQN}$
\STATE Update $\Sigma_a$ according to equation (\ref{sigmaUpdate})
\STATE Take action $a'$
\STATE $s\leftarrow s', a\leftarrow a'$
   \ENDFOR
   \ENDFOR
\end{algorithmic}
\end{algorithm}

\subsection{DQN-URBE}
Since the URBE algorithm requires solving a robust optimization problem at each episode, it is computationally costly and not scalable. Therefore, we present our DQN-URBE algorithm (Algorithm \ref{dqn-urbe}), which avoids this problem by keeping the uncertainty set fixed and finite but adds the robust local uncertainty as an exploration bonus. 

The robust Bellman equation utilizes a robust TD-error as a loss criterion for learning a minimax policy \citep{modelmismatch, kalman}. The robust TD error to be minimized is defined as: 
\begin{equation*}
\begin{split}
  \delta^h :=  r(s^h, a^h) &+ \gamma \inf_{p\in \mathcal{P}} \sum_{s'\in\mathcal{S}} p(s^h, a^h, s')\max_{a'\in\mathcal{A}}  Q(s', a') \\
  &- Q(s^h, a^h),  
  \end{split}
\end{equation*}
where the uncertainty set is fixed. Works \citep{robustOptions, SRAC} used this method in deep robust RL and considered a finite uncertainty set of models. The resulting performance has been shown to lead to robust yet overly conservative behavior.

In order to generate a less conservative solution, DQN-URBE 
takes the posterior variance of the robust Q-values into account. We should note that several of the assumptions that have been made and used for estimating the robust local uncertainty are being violated in deep settings. Indeed, transition models are no longer acyclic, the policy we estimate the posterior variance on is no longer fixed, and URBE is not solved exactly but approximated by a sub-network of the robust Q-network. However, this heuristic approach works well in practice, as it will be shown in the next section. 
\begin{figure}[!!h]
\centering
\includegraphics[width = 1.0\columnwidth]{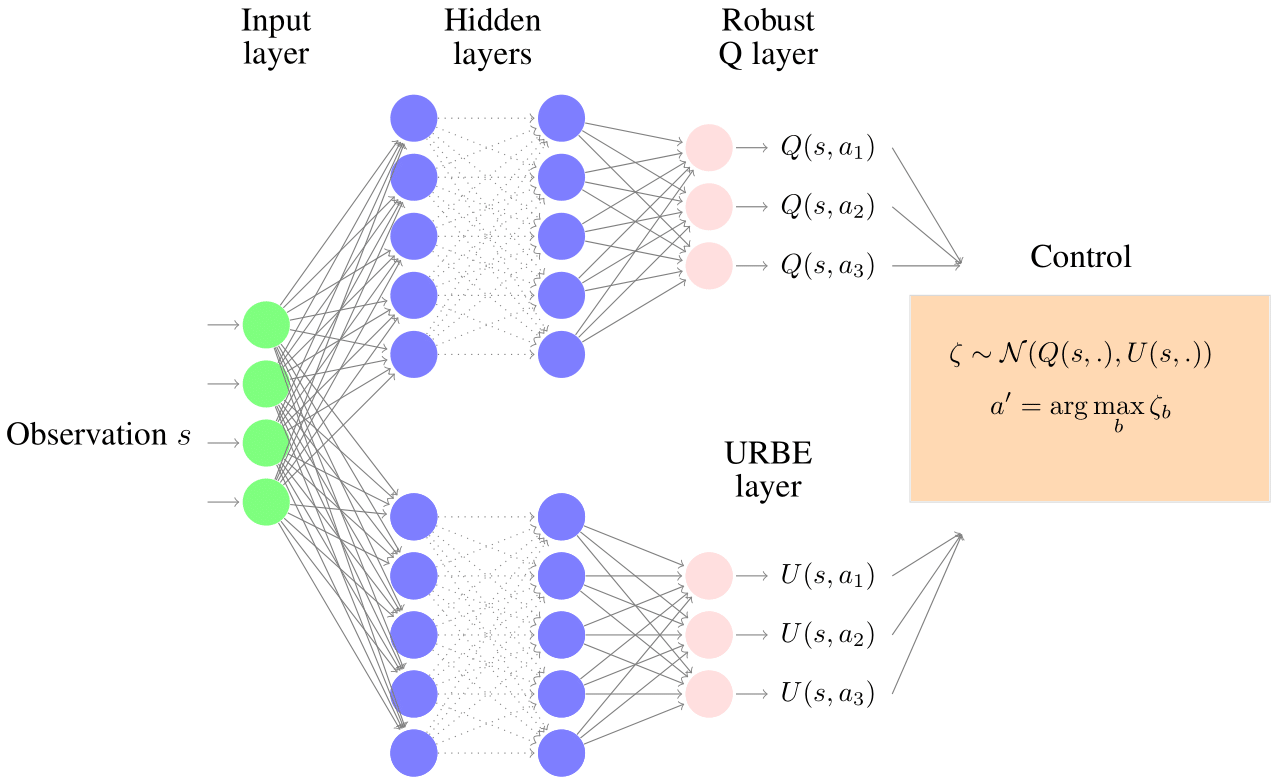}
\caption{DQN-URBE architecture}
\label{nn-urbe}
\end{figure}

DQN-URBE consists of a neural network architecture that has two output heads, as shown in Figure \ref{nn-urbe}. The first head attempts to learn the optimal robust Q-function of a fixed uncertainty set via the robust-DQN subroutine described in \citep{kalman}. It is similar to regular DQN except that it utilizes the robust TD-error as a loss criterion. The other head attempts to estimate the robust uncertainty for the robust Q-function, as mentioned in Section \ref{posterior-variance-section}. The robust local uncertainty is estimated using the function approximation method described in Section \ref{robust-local-uncertainty}. This defines the loss function to minimize for learning the robust local uncertainty parameters. We added stop-gradients to prevent the posterior variance from affecting the robust Q-network parameters and vice-versa. At each step, the agent acts greedily with respect to the robust Q-function plus the robust local uncertainty to encourage exploration.

\begin{figure*}[htp]
	\centering
	
	\subfigure[]{\label{simplemdp}
	\resizebox {.3\linewidth} {!} {
	\begin{tikzpicture}[scale = 1]
  \node[state] at (3,0) (0) {$s_0$};
  \node[state] at (3,3) (1) {$s_1$};
  \node[state] at (6,1.5) (2) {$s_2$};
  \node[state] at (9,2.5) (3) {$s_3$};
  \node[state] at  (6,0) (4) {$s_4$};
  \node[state] at (6,-1.5) (5) {$s_5$};
  \node[state] at (9,-2) (6) {$s_6$};

  \draw[every loop]
  	(0) edge[] node[below, pos = .5, xshift = 3.5mm] {$a_1$} (1)
	(0) edge node[above, pos = 0.5] {$a_2$} (2)
	(0) edge node[above, pos = 0.5] {$a_3$} (4)
	(0) edge[] node[above, pos = 0.5] {$a_4$} (5)

	(2) edge [dashed] node[above, pos = 0.5] {} (4)
	(4) edge [dashed] node[above, pos = 0.5] {} (5)
	(5) edge [dashed] node[above, pos = 0.5] {} (6)
	
	(2) edge[] node[above, pos = 0.6,yshift = 1mm] {} (3)
	(4) edge[] node[above, pos = 0.6,yshift = 1mm] {} (3)
	(5) edge[] node[above, pos = 0.6,yshift = 1mm] {} (3)
	
	(3) edge[bend right] node[above, pos = 0.1] {\mbox{\Large $R_{\text{good}}$}} (0)
	(6) edge[bend left] node[below, pos = 0.1, yshift = -1mm] {\mbox{\Large $R_{\text{bad}}$}} (0)
	(1) edge[bend right] node[above, pos = .1, xshift = -9mm]{\mbox{\Large $R_{\text{minimax}}$}} (0);
\end{tikzpicture}}}
	\subfigure[]{\label{simplemdp-performance}
	\includegraphics[width=0.3\linewidth,height=0.7\textheight,keepaspectratio]{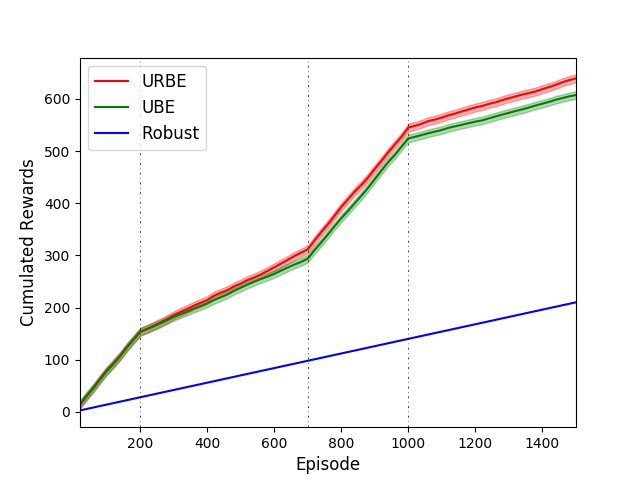}}
	\subfigure[]{\label{marsroverGrid}
	\includegraphics[width=0.3\linewidth,height=0.3\textheight,keepaspectratio]{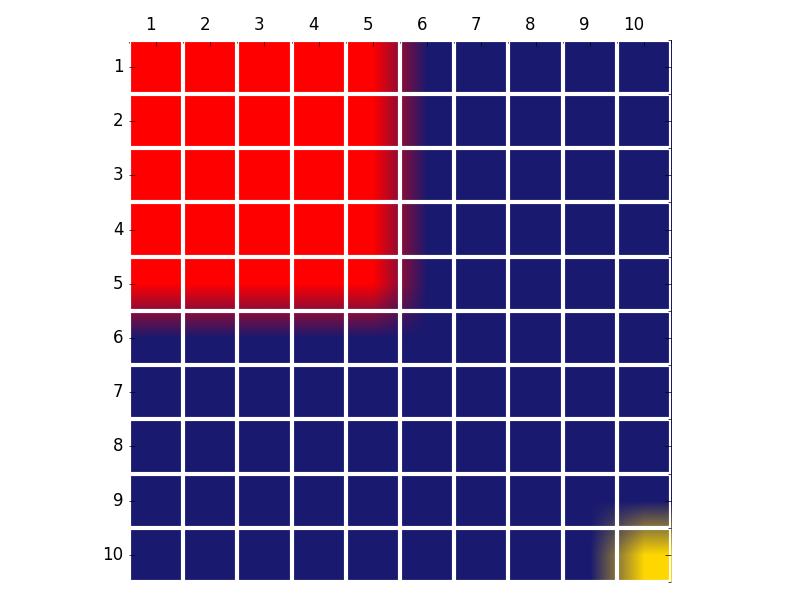}}
	
	\caption{(a) Simple MDP illustration, with initial and terminating state $s_0$. (b) Comparison of the accumulated rewards on the simple MDP. The vertical lines mark changing dynamics. (c) Mar's Rover domain. Starting states are randomly chosen from the red zone. The goal state is in orange.} 
\end{figure*}
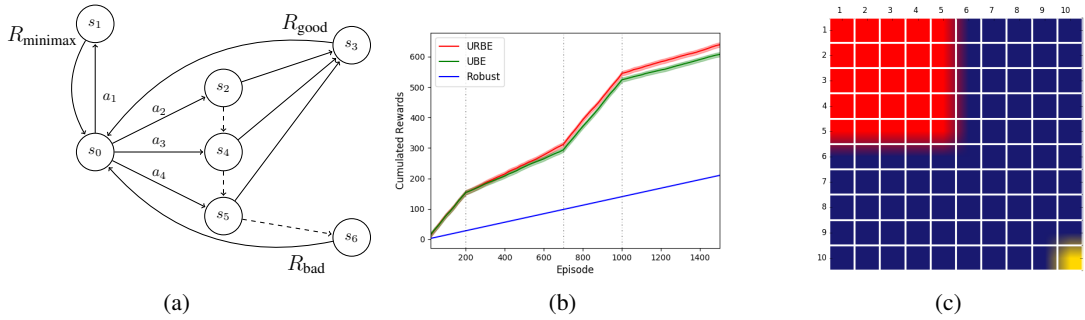

\section{EXPERIMENTS}
\label{experiments-section}
In this section, we test the performance of our URBE-based approach on three different domains: a toy MDP, a Mar's Rover domain and Cartpole. We first execute URBE on a toy MDP and analyze its performance under changing dynamics. We then propose DQN-URBE, a deep RL algorithm that scales our URBE approach to higher dimensional domains. We run and analyze the performance of DQN-URBE on a Mar's rover domain and Cartpole. In each case, we compare the robust DQN-URBE policy to three baselines: (1) a vanilla DQN; (2) an overly conservative, robust DQN agent and (3) a DQN that uses UBE for exploration \citep{UBE}. Neural network structures and hyperparameters can be found in the Appendix.

\subsection{SIMPLE MDP}
We consider a variant of the 7-state MDP introduced in \citep{OLRMfull}, which is illustrated in Figure \ref{simplemdp}. The agent starts from state $s_0$ and chooses one of $4$ actions. Action $a_1$ leads to a purely deterministic outcome, whereas $a_2, a_3$ and $a_4$ may be subject to adversarial transitions and lead the agent to either state $s_3$ or $s_6$. The agent is brought back to the terminating state $s_0$ once it reached $s_1$, $s_3$ or $s_6$. These latter states are the only ones with non-zero rewards, although in practice, we set $R(s_6) =0$, $R(s_1) = 0.14$ and $R(s_3) = 1$. 

This MDP captures the main characteristics of a grid-world domain in which the agent must reach a gold state under adversarial transitions. At any episode, an adversary can choose any transition probability $p(s_3)$ for the agent to reach the gold state $s_3$ from $s_2, s_4$ or $s_5$. If it behaves nicely and $p(s_3) = 1$, the agent can achieve maximal reward. However, if $p(s_3) = 0$, the agent is brought to the "bad state" $s_6$ in case it did not choose action $a_1$, and thus gets minimal reward. For a fixed uncertainty that accounts for all adversarial transitions, a minimax-optimal policy corresponds to constantly taking action $a_1$.

Figure \ref{simplemdp-performance} shows the accumulated rewards over running time for the described MDP, and each vertical line marks a change in the adversarial probability $p(s_3)$. We successively set such probability to $0.001, 0.8, 0.1$ and  $0.9$. Cumulative rewards have been averaged over $10$ runs for UBE and URBE, whereas the performance of the robust policy is deterministic. As we can see in the figure, the robust agent is overly conservative although its reward is stable under adversarial transitions. Also, since the UBE-based agent does not account for adversarial transitions, it performs worse than URBE. 

\subsection{MAR'S ROVER}
We extend the size of the previous MDP and consider a $10\times 10$ grid-world domain inspired by \citep{chowCVAR, RCPO}. A rover starts at a random state from the top left of the grid (red zone in Figure \ref{marsroverGrid}) and is required to travel to the goal located in the bottom right corner (orange square in Figure \ref{marsroverGrid}) in less than $200$ steps so as to get a high reward $R_{\text{success}}$. The transition is stochastic. On each step, if it chooses to move towards the goal, the agent may be brought back to a final state and get a negative reward $R_{\text{fail}}$ with probability $p$. Otherwise, it moves into the chosen direction and receives a small negative reward $R_{\text{step}}$. 

We trained vanilla DQN, robust DQN, DQN-UBE and DQN-URBE on a nominal probability of failure $p = 0.005$. Uncertainty sets were generated by sampling $15$ probabilities $p$ in $(0,1)$. Figure \ref{marsrover} shows the testing performance of each strategy over different probabilities. We see that the robust agent is unable to reach the goal state, even on the nominal model. However, it is never brought back to the failing state but rather avoids moving towards the goal state, which explains its stable performance. Similarly to vanilla DQN, the UBE agent performs well on the nominal but is most sensitive to changing dynamics. Its reward gets even worse than robust DQN above $p = 0.2$, since it tries to move towards the goal but is barred by adversarial transitions. During testing, DQN-URBE reaches high reward on the nominal model and shows less sensitivity to increasing probabilities. It is therefore less conservative than robust DQN but stays robust to model misspecification.

We further investigated the trajectories of three agents across the grid, under appropriate dynamics. A number of $100$ testing episodes were run. Figure 4 represents heatmaps of states that were attained, with their proportion of visits. These are visualized in four colors, ranging from the lowest to the highest proportion: dark blue, cyan, yellow and brown. Figures \ref{marsRover-robust} and \ref{marsRover-urbe} correspond to testing episodes on the nominal model $p = 0.005$ for robust DQN and URBE, respectively. The robust agent never reaches the goal, while URBE shows high proportion of visitation on the winning state. In Figures \ref{marsRover-ube-change} and \ref{marsRover-urbe-change}, a higher probability of failure ($p = 0.2$) has been used to test the robustness of DQN-UBE against DQN-URBE. The URBE agent clearly shows more robustness than UBE, as it reaches the goal state under this mispecified model.

\begin{figure}[!!h]
\centering
\includegraphics[width = 1.0\columnwidth]{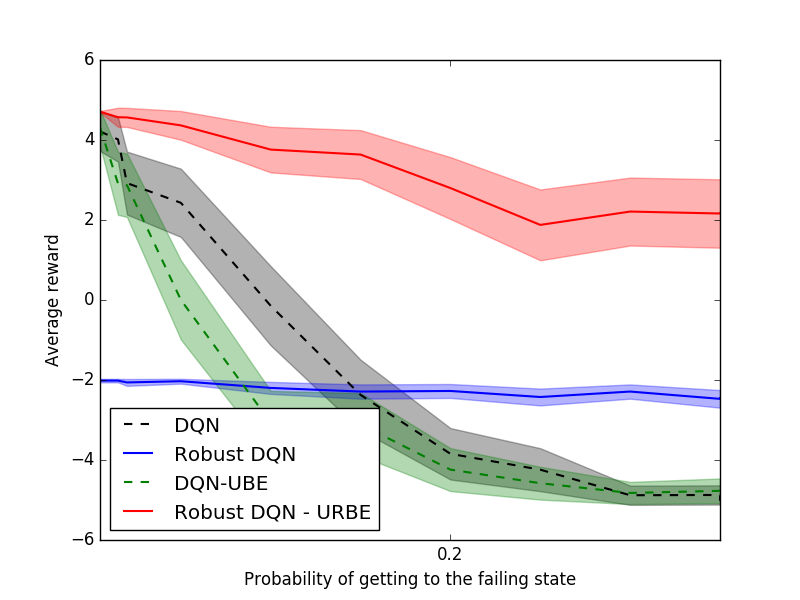}
\caption{Testing rewards on Mar's Rover.}
\label{marsrover}
\end{figure}

\begin{figure*}[htp]
	\centering
	\label{heatmaps}
	\subfigure[]{\label{marsRover-robust}
	\includegraphics[width=0.24\linewidth,height=0.5\textheight,keepaspectratio]{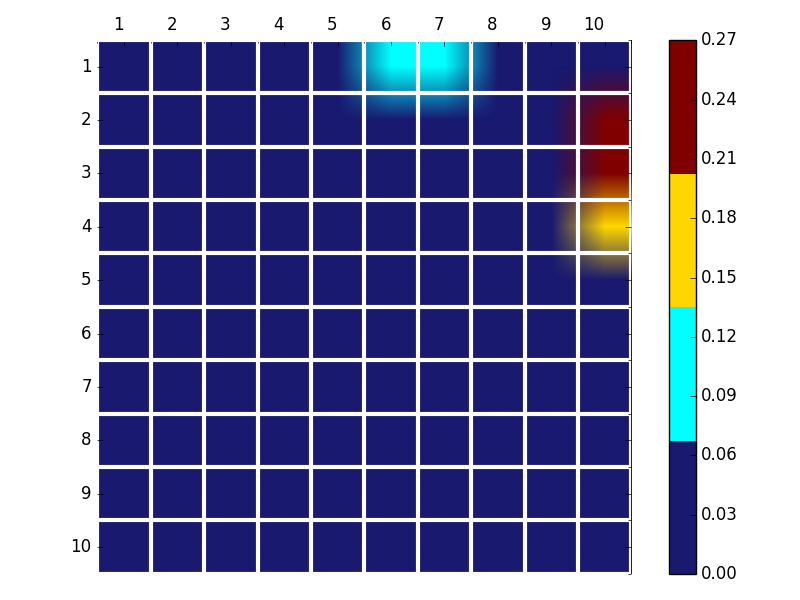}}
	\subfigure[]{\label{marsRover-urbe}
	\includegraphics[width=0.24\linewidth,height=0.5\textheight,keepaspectratio]{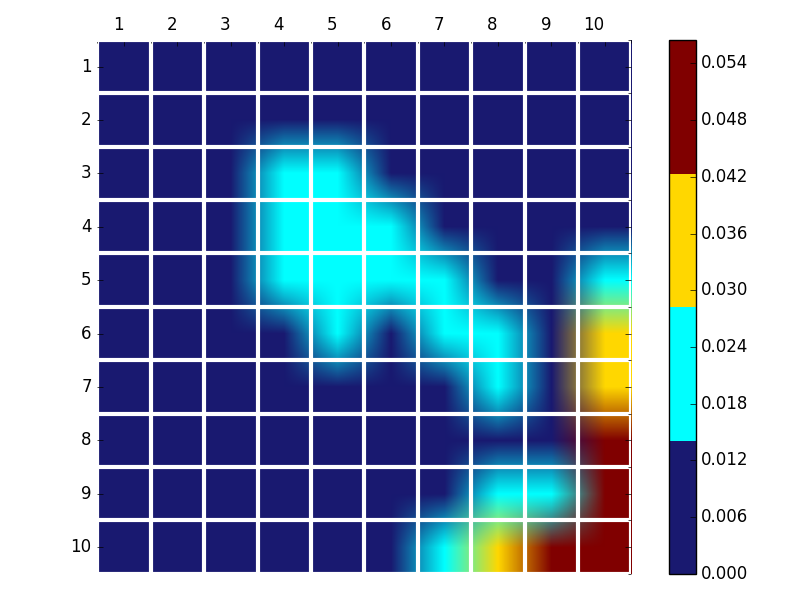}}
	\subfigure[]{\label{marsRover-ube-change}
	\includegraphics[width=0.24\linewidth,height=0.5\textheight,keepaspectratio]{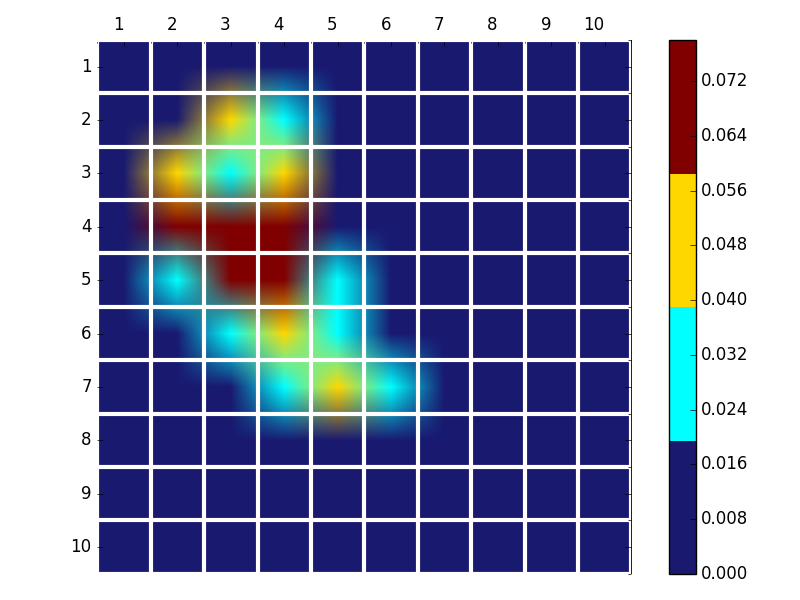}}
	\subfigure[]{\label{marsRover-urbe-change}
	\includegraphics[width=0.24\linewidth,height=0.5\textheight,keepaspectratio]{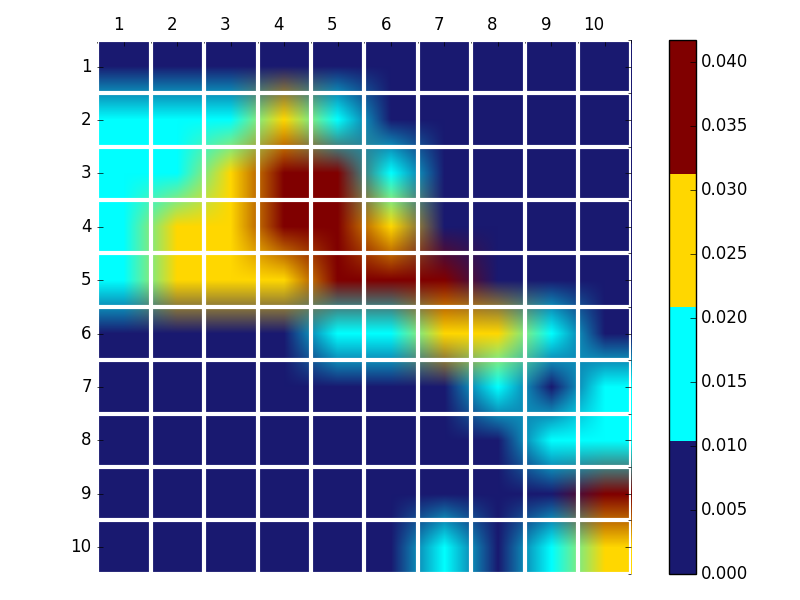}}
	\caption{Mar's Rover heatmaps of state visitations during $100$ testing episodes. (a) Robust DQN on $p = 0.005$. (b) DQN-URBE on $p = 0.005$. (c) DQN-UBE on $p = 0.2$. (d) DQN-URBE on $p = 0.2$. Robust DQN is too conservative compared to URBE, while UBE is less robust than URBE.} 
\end{figure*}


\subsection{CARTPOLE}
In Cartpole, the agent's goal is to balance a pole atop a cart in a vertical position. The system corresponds to a continuous MDP where each state is a 4-tuple $\langle x, \dot{x},\theta, \dot{\theta} \rangle$
representing the cart position, the cart speed, the pole angle with respect to the vertical and its angular speed respectively.   
The agent can make one of two actions: apply a constant force either to the right or to the left of the pole. It gets 
a reward of $1$ if the pole has not fallen down and if it stayed in the boundary sides of the screen. If it
terminates, the agent receives a reward of $0$. Each episode lasts for $200$ steps.

\begin{figure}[!!h]
\centering
\includegraphics[width = 1.0\columnwidth]{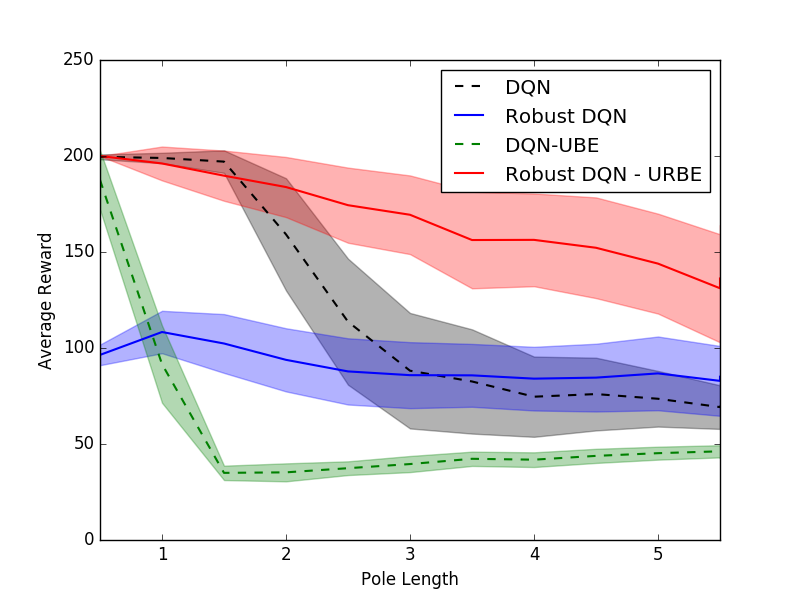}
\caption{Testing rewards on Cartpole.}
\label{cartpole-testRwd}
\end{figure}

All agents have been trained on a nominal pole length of $0.75$. Uncertainty sets were generated by sampling $15$ lengths from a normal distribution centered at the nominal. The agents were then tested over $200$ episodes on different pole lengths. Figure \ref{cartpole-testRwd} shows their average reward during testing. Robust DQN is overly conservative on the nominal length although it stays robust to model misspecification, compared to DQN-UBE which is most sensitive to changing pole lengths. On the other hand, DQN-URBE shows the best trade-off between less conservativeness on the nominal and robustness to higher lengths. This leads it to perform best on the nominal length and on higher ones as well.

\begin{figure*}[htp]
\label{testingPerformance}
	\centering
	\subfigure[]{\label{cartpole-urbe}
	\includegraphics[width=0.3\linewidth,height=0.5\textheight,keepaspectratio]{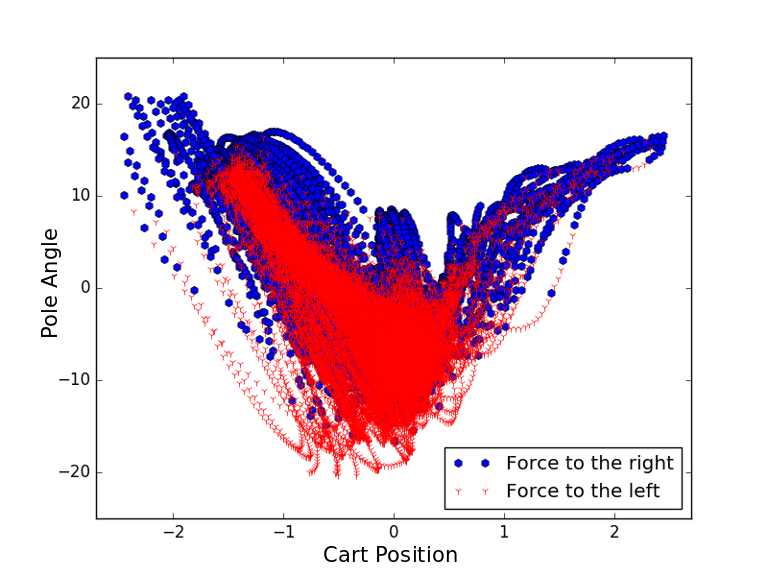}}
	\subfigure[]{\label{cartpole-robust}
	\includegraphics[width=0.3\linewidth,height=0.5\textheight,keepaspectratio]{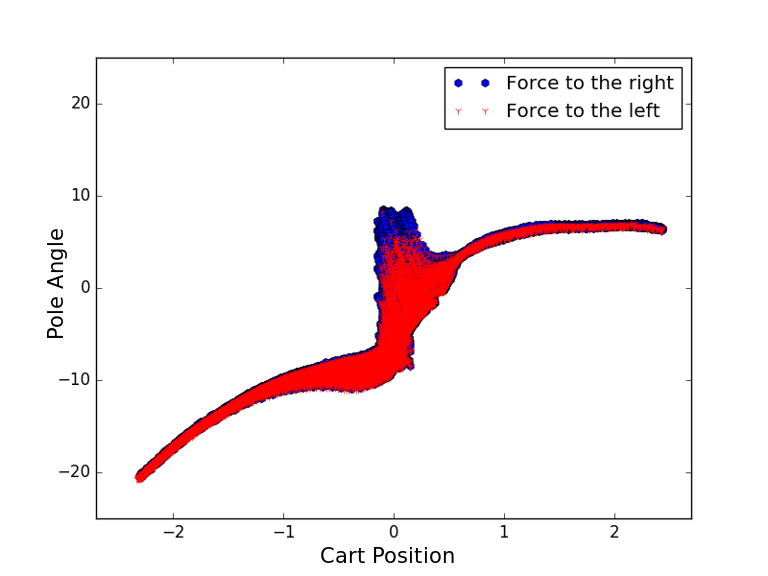}}
	\subfigure[]{\label{changedyn-cartpole}
	\includegraphics[width=0.3\linewidth,height=0.5\textheight,keepaspectratio]{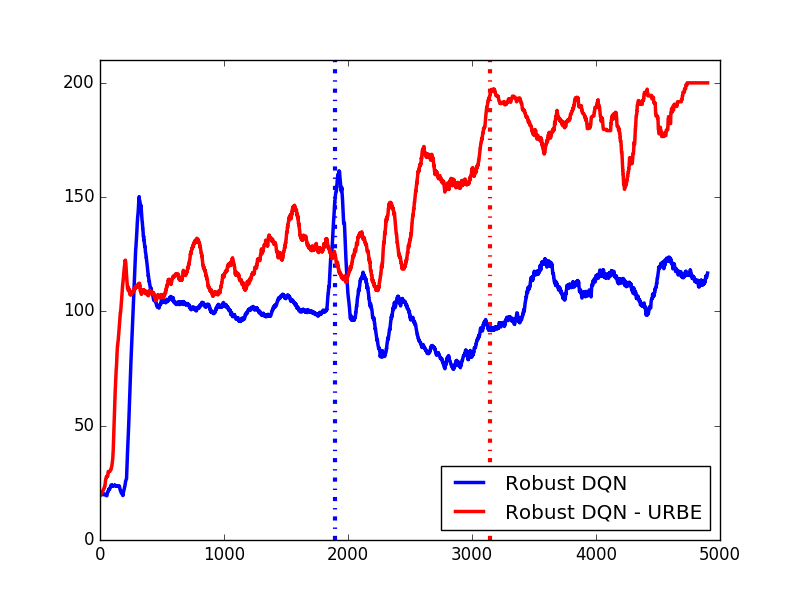}}
	\caption{(a-b) Cartpole on the nominal model. These figures show the states $\langle x, \theta\rangle$ attained and the two colors correspond to the action applied on these states. (a) DQN-URBE; (b) Robust DQN. (c) Training score of DQN-URBE and robust DQN. Vertical lines mark changes of the pole length. DQN-URBE explores more than robust DQN while it stays robust to changing dynamics.} 
\end{figure*}

In order to further test the exploration capacity of the robust agents, we also compared the projected states $\langle x, \theta \rangle$ attained on the nominal model during $200$ testing episodes. Figures \ref{cartpole-urbe} and \ref{cartpole-robust} show that the span of visited states for the robust agent is quite limited compared to URBE, which may explain its overly conservative behavior. We also compared the sensitivity of robust DQN with that of DQN-URBE to changing dynamics during training. In practice, we waited for both agents to converge before changing the pole length from $0.75$ to $1.25$. We also augmented the number of training episodes from $4000$ to $5000$. Figure \ref{changedyn-cartpole} shows a smoothed version of the training curve. Changes in dynamics are marked by vertical lines, and each color is that of the corresponding agent. As can be seen in the figure, although URBE converges more slowly, it recovers much faster than robust DQN which does not recover at all to its optimal reward. Moreover, URBE is able to reach maximal reward.

\begin{table*}[!h]
\caption{Comparison of previous approaches with URBE}
\label{tab}
\vskip 0.15in
\begin{center}
\begin{small}
\begin{sc}
\begin{tabular}{lcccr}
\toprule

\strut \textbf{Reference} & \vtop{\hbox{\strut \textbf{Mitigates RMDPs}}\hbox{\strut \textbf{Conservativeness}}} & \vtop{\hbox{\strut \textbf{Online Learning of}}\hbox{\strut \textbf{an Uncertainty Set}}} & \textbf{Scalability} \\
\midrule
URBE (this paper)    & $\surd$ &  $\surd$  & $\surd$ \\
\citep{goyalRectangularity}      & $\surd$&  $\times$& $\times$\\
\citep{petrikNonrectangular}    &  $\surd$ & $\times$ & $\surd$ \\
\citep{petrikBayes2}    &  $\surd$ & $\times$ & $\times$ \\
\citep{SRAC}    & $\surd$&  $\times$&     $ \surd$   \\
\citep{OLRMfull} & $\surd$ &  $\surd$ & $\times$\\
\citep{LDST,krectangularity}      & $\surd$&  $\times$& $\times$\\
\citep{DRMDP1, DRMDP2}      & $\surd$&  $\times$& $\times$\\
\bottomrule
\end{tabular}
\end{sc}
\end{small}
\end{center}
\vskip -0.1in
\end{table*}

\section{RELATED WORK}
Several methods have been proposed to mitigate conservativeness in robust RL. These are indicated in Table \ref{tab}. These works can be collectively grouped into three distinct approaches for mitigating the conservativeness of RMDPs. The first one focuses on circumventing rectangular uncertainty sets. This includes works \citep{LDST, krectangularity} that consider coupled uncertainties which still lead to tractable solutions for the robust RL problem. Similarly, in \citep{petrikNonrectangular}, the authors proposed the construction of non-rectangular uncertainty sets that take advantage of transfer knowledge between states. More recently, the work \citep{goyalRectangularity} introduced an uncertainty set structured as an ambiguous linear function of a factor matrix and showed that the underlying minimax policy is computationally tractable. 

A second approach for overcoming conservativeness of robust RL is to consider a distribution over the uncertainty set rather than its worst-case model. In \citep{DRMDP1, DRMDP2}, structural information on parameter distribution is assumed. It is used for deriving an optimal policy under the worst parameter distribution using distributionally robust optimization. In \citep{SRAC}, the authors showed that if we fix such a distribution, the corresponding optimal policy interpolates between being aggressive and robust. Nonetheless, as indicated in the second column of Table \ref{tab}, all of these approaches address the problem of planning in robust MDPs or its variants without learning the uncertainty set in an online manner. Therefore, they cannot adapt it to changing dynamics.

Conversely, a third approach involves learning adversarial transitions and/or rewards. The work \citep{OLRMfull} considered the problem of learning the uncertainty set in a frequentist setting and used OFU methods for detecting adversarial states and updating the uncertainty set accordingly. Although the resulting algorithm is online and provably efficient, it is not scalable to large domains and its efficiency strongly relies on the statistical analysis. Using a Bayesian setting in \citep{petrikBayes2}, the authors addressed an algorithm that constructs Bayesian uncertainty sets in a safe manner. However, besides learning an uncertainty set offline with a \textbf{fixed} batch of data, their method is not scalable because of its computational cost. In contrast, our proposed approach aims to adapt the level of robustness iteratively and online from a dynamic stream of data.

\section{CONCLUSION}
We presented a Bayesian approach to learning less conservative solutions when solving Robust MDPs. This is achieved using the Uncertainty Robust Bellman Equation (URBE), our adaptation of the UBE equation, which encourages safe exploration and implicitly modifies the uncertainty set online using new observations. We scale this approach to higher dimensional domains using the DQN-URBE algorithm and show the ability of the agent to learn less conservative solutions in a toy MDP, a Mar's rover domain and Open AI gym's Cartpole domain. Finally, we show the ability of the agent to adapt to changing dynamics  significantly faster than a robust DQN agent during training. Our approach shed light on the advantages of adding a variance bonus to robust Q-learning for encouraging safe exploration in lowering the conservativeness of robust strategies. Further work should analyze the asymptotic behavior of our URBE-based method as well as the impact of the size of the posterior uncertainty set on the posterior variance of robust Q-values.

\subsubsection*{Acknowledgements}
The authors would like to thank Chen Tessler for his help and useful comments on this work. E. Derman would
like to acknowledge the support of the Israel Science Foundation grant. 

\bibliographystyle{siam}
\bibliography{URBEbib}

\begin{thebibliography}{35}
\providecommand{\natexlab}[1]{#1}
\providecommand{\url}[1]{\texttt{#1}}
\expandafter\ifx\csname urlstyle\endcsname\relax
  \providecommand{\doi}[1]{doi: #1}\else
  \providecommand{\doi}{doi: \begingroup \urlstyle{rm}\Url}\fi

\bibitem[Berger(2013)]{Berger}
{James O.} Berger.
\newblock \emph{Statistical Decision Theory and Bayesian Analysis}.
\newblock Springer Science and Business Media, 2013.

\bibitem[Bertsekas(2000)]{bertsekas}
{Dimitri P.} Bertsekas.
\newblock \emph{Dynamic Programming and Optimal Control}, volume~1.
\newblock Athena Scientific, 2 edition, 2000.

\bibitem[Chow et~al.(2015)Chow, Tamar, Mannor, and Pavone]{chowCVAR}
Yinlam Chow, Aviv Tamar, Shie Mannor, and Marco Pavone.
\newblock Risk-sensitive and robust decision-making: a {CV}a{R} optimization
  approach.
\newblock \emph{Advances in Neural Information Processing Systems}, pages
  1522--1530, 2015.

\bibitem[Dearden et~al.(1998)Dearden, Friedman, and Russel]{deardenMF}
Richard Dearden, Nir Friedman, and Stuart Russel.
\newblock Bayesian q-learning.
\newblock \emph{AAAI}, 1998.

\bibitem[Dearden et~al.(1999)Dearden, Friedman, and Andre]{deardenMB}
Richard Dearden, Nir Friedman, and David Andre.
\newblock Model-based bayesian exploration.
\newblock \emph{UAI}, pages 150--159, 1999.

\bibitem[Derman et~al.(2018)Derman, Mankowitz, and Mann]{SRAC}
Esther Derman, {Daniel J.} Mankowitz, and {Timothy A.} Mann.
\newblock Soft-robust actor-critic policy-gradient.
\newblock \emph{AUAI press for Association for Uncertainty in Artificial
  Intelligence}, pages 208--218, 2018.

\bibitem[{Di-Castro Shashua} and Mannor(2017)]{kalman}
Shirli {Di-Castro Shashua} and Shie Mannor.
\newblock Deep {R}obust {K}alman {F}ilter.
\newblock \emph{arXiv preprint arXiv:1703.02310v1}, 2017.

\bibitem[Ghavamzadeh et~al.(2015)Ghavamzadeh, Mannor, Pineau, and Tamar]{BRLS}
Mohammad Ghavamzadeh, Shie Mannor, Joelle Pineau, and Aviv Tamar.
\newblock Bayesian {R}einforcement {L}earning: {A} {S}urvey.
\newblock \emph{Foundations and Trends in Machine Learning}, 8\penalty0
  (5-6):\penalty0 359--492, 2015.

\bibitem[Golub and {Van Loan}(1996)]{golub}
{Gene H.} Golub and {Charles F.} {Van Loan}.
\newblock \emph{Matrix Computations}.
\newblock The John Hopkins University Press, 1996.

\bibitem[Goyal and Grand-Clement(2019)]{goyalRectangularity}
Vineet Goyal and Julien Grand-Clement.
\newblock Robust {M}arkov decision process: Beyond rectangularity.
\newblock \emph{arXiv preprint arXiv:1811.00215v4}, 2019.

\bibitem[Gupta(2018)]{guptaDRO}
Vishal Gupta.
\newblock Near-optimal {B}ayesian ambiguity sets for distributionally robust
  optimization.
\newblock \emph{Management Science}, 2018.

\bibitem[Iyengar(2005)]{iyengar}
Garud~N. Iyengar.
\newblock Robust {D}ynamic {P}rogramming.
\newblock \emph{Mathematics of Operations Research}, 30\penalty0 (2):\penalty0
  257--280, 2005.

\bibitem[Jaksch et~al.(2010)Jaksch, Ortner, and Auer]{UCRL2}
Thomas Jaksch, Ronald Ortner, and Peter Auer.
\newblock Near-optimal regret bounds for reinforcement learning.
\newblock \emph{Journal of Machine Learning Research}, 11:\penalty0 1563--1600,
  2010.

\bibitem[Lim et~al.(2016{\natexlab{a}})Lim, Xu, and Mannor]{OLRM-NIPS}
Shiau~Hong Lim, Huan Xu, and Shie Mannor.
\newblock Reinforcement {L}earning in {R}obust {M}arkov {D}ecision {P}rocesses.
\newblock \emph{Mathematics of Operations Research}, 41\penalty0 (4):\penalty0
  1325--1353, 2016{\natexlab{a}}.

\bibitem[Lim et~al.(2016{\natexlab{b}})Lim, Xu, and Mannor]{OLRMfull}
{Shiau Hong} Lim, Huan Xu, and Shie Mannor.
\newblock Reinforcement learning in robust {M}arkov decision processes.
\newblock \emph{Mathematic of Operations Research}, 41\penalty0 (4):\penalty0
  1325--1353, November 2016{\natexlab{b}}.

\bibitem[Mankowitz et~al.(2018)Mankowitz, Mann, Mannor, Precup, and
  Bacon]{robustOptions}
Daniel~J Mankowitz, Timothy~A Mann, Shie Mannor, Doina Precup, and Pierre-Luc
  Bacon.
\newblock Learning {R}obust {O}ptions.
\newblock In \emph{AAAI}, 2018.

\bibitem[Mannor et~al.(2007)Mannor, Simester, Sun, and Tsitsiklis]{biasMannor}
Shie Mannor, Duncan Simester, Peng Sun, and {John N.} Tsitsiklis.
\newblock Bias and {V}ariance {A}pproximation in {V}alue {F}unction
  {E}stimates.
\newblock \emph{Management Science}, 53\penalty0 (2):\penalty0 308--322, 2007.

\bibitem[Mannor et~al.(2012)Mannor, Mebel, and Xu]{LDST}
Shie Mannor, Ofir Mebel, and Huan Xu.
\newblock Lightning {D}oes {N}ot {S}trike {T}wice: {R}obust {MDP}s with
  {C}oupled {U}ncertainty.
\newblock In \emph{ICML}, 2012.

\bibitem[Mannor et~al.(2016)Mannor, Mebel, and Xu]{krectangularity}
Shie Mannor, Ofir Mebel, and Huan Xu.
\newblock Robust {MDP}s with k-{R}ectangular {U}ncertainty.
\newblock \emph{Mathematics of Operations Research}, 41\penalty0 (4):\penalty0
  1484--1509, 2016.

\bibitem[Nilim and {El Ghaoui}(2005)]{nilim}
Arnab Nilim and Laurent {El Ghaoui}.
\newblock Robust control of {M}arkov decision processes with uncertain
  transition matrices.
\newblock \emph{Operations Research}, 53\penalty0 (5):\penalty0 783--798, 2005.

\bibitem[O'Donoghue et~al.(2018)O'Donoghue, Osband, Munos, and Mnih]{UBE}
Brendan O'Donoghue, Ian Osband, Remi Munos, and Volodymyr Mnih.
\newblock The {U}ncertainty {B}ellman {E}quation and {E}xploration.
\newblock \emph{ICML}, 2018.

\bibitem[Osband and {Van Roy}(2017)]{whyPS}
Ian Osband and Benjamin {Van Roy}.
\newblock Why is posterior sampling better than optimism for reinforcement
  learning?
\newblock \emph{ICML}, pages 2701--2710, 2017.

\bibitem[Osband et~al.(2013)Osband, {Van Roy}, and Russo]{PSRL}
Ian Osband, Benjamin {Van Roy}, and Daniel Russo.
\newblock (more) efficient reinforcement learning via posterior sampling.
\newblock \emph{NIPS}, pages 3003--3011, 2013.

\bibitem[Osband et~al.(2016)Osband, {Van Roy}, and
  Wen]{randomizedValueFunctions}
Ian Osband, Benjamin {Van Roy}, and Zheng Wen.
\newblock Generalization and exploration via randomized value functions.
\newblock \emph{Proceedings of The 33rd International Conference on Machine
  Learning}, 48:\penalty0 2377--2386, 2016.

\bibitem[Petrik and Russell(2019)]{petrikBayes2}
Marek Petrik and {Reazul Hasan} Russell.
\newblock Beyond confidence regions: Tight bayesian ambiguity sets for robust
  mdps.
\newblock \emph{arXiv preprint arXiv:1902.07605}, 2019.

\bibitem[Roy et~al.(2017)Roy, Xu, and Pokutta]{modelmismatch}
Aurko Roy, Huan Xu, and Sebastian Pokutta.
\newblock Reinforcement learning under {M}odel {M}ismatch.
\newblock \emph{31st Conference on Neural Information Processing Systems},
  2017.

\bibitem[Russel and Petrik(2018)]{petrikBayes}
{Reazul Hasan} Russel and Marek Petrik.
\newblock Tight bayesian ambiguity sets for robust {MDP}s.
\newblock \emph{Neural Information Processing Systems}, 2018.

\bibitem[Strens(2000)]{strens}
Malcolm Strens.
\newblock A bayesian framework for reinforcement learning.
\newblock \emph{ICML}, 2000.

\bibitem[Tamar et~al.(2014)Tamar, Mannor, and Xu]{scalingRMDP}
Aviv Tamar, Shie Mannor, and Huan Xu.
\newblock Scaling up robust {MDP}s using function approximation.
\newblock \emph{ICML}, 32:\penalty0 1401--1415, 2014.

\bibitem[Tessler et~al.(2019)Tessler, Mankowitz, and Mannor]{RCPO}
Chen Tessler, {Daniel J.} Mankowitz, and Shie Mannor.
\newblock Reward constrained policy optimization.
\newblock \emph{ICLR}, 2019.

\bibitem[Thompson(1933)]{TS}
{William R.} Thompson.
\newblock On the likelihood that one unknown probability exceeds another in
  view of the evidence of two samples.
\newblock \emph{Biometrika}, 25\penalty0 (3-4):\penalty0 285--294, December
  1933.

\bibitem[Tirinzoni et~al.(2018)Tirinzoni, Petrik, Chen, and
  Ziebart]{petrikNonrectangular}
Andrea Tirinzoni, Marek Petrik, Xiangli Chen, and Brian Ziebart.
\newblock Policy-conditioned uncertainty sets for robust {M}arkov decision
  processes.
\newblock \emph{Advances in Neural Information Processing Systems}, pages
  8953--8963, 2018.

\bibitem[Wiesemann et~al.(2013)Wiesemann, Kuhn, and Rustem]{wieseman}
Wolfram Wiesemann, Daniel Kuhn, and Ber{\c c} Rustem.
\newblock Robust {M}arkov decision processes.
\newblock \emph{Mathematics of Operations Research}, 38\penalty0 (1):\penalty0
  153--183, February 2013.

\bibitem[Xu and Mannor(2012)]{DRMDP1}
Huan Xu and Shie Mannor.
\newblock Distributionally {R}obust {M}arkov {D}ecision {P}rocesses.
\newblock \emph{Mathematics of Operations Research}, 37\penalty0 (2):\penalty0
  288--300, 2012.

\bibitem[Yu and Xu(2016)]{DRMDP2}
Pengqian Yu and Huan Xu.
\newblock Distributionally {R}obust {C}ounterpart in {M}arkov {D}ecision
  {P}rocesses.
\newblock \emph{IEEE Transactions on Automatic Control}, 61\penalty0
  (9):\penalty0 2538 -- 2543, 2016.

\end{thebibliography}

\appendix
\onecolumn
\section*{\hfil A Bayesian Approach to Robust Reinforcement Learning - Appendix \hfil}

\section{Theoretical Proofs}

Recall the assumptions made in the paper:
\begin{assumption}
\label{acyclic2}
For any episode, the graph resulting from a worst-case transition model is directed and acyclic.
\end{assumption}

\begin{assumption}
\label{Rbound2}
For all $(s,a)\in\mathcal{S}\times\mathcal{A}$, the rewards are bounded: $-R_{\max} \leq r_{sa}\leq R_{\max}$. This implies that the robust Q-value is bounded as well: $\mid Q_{sa}^h\mid \leq H R_{\max} =: Q_{\max}$.
\end{assumption}

Recall also the worst-case transition from a posterior uncertainty set:
\begin{equation*}
\begin{split}
  \widehat{Q}_{sa}^h &= r_{sa}^h + \gamma\inf_{p \in \widehat{\mathcal{P}}_{sa}^h}\sum_{s', a'}\pi_{s'a'}^h p_{sas'} \widehat{Q}_{s'a'}^{h+1} \enspace ,
  \end{split}
\end{equation*}
with $\widehat{Q}_{sa}^{H+1} = 0$ and
\begin{equation}
\label{p-post2}
\begin{split}
  \widehat{p}_{sa}^h &\in \arg \min_{p \in \widehat{\mathcal{P}}_{sa}^h}\sum_{s', a'}\pi_{s'a'}^h p_{sas'} \widehat{Q}_{s'a'}^{h+1}
\end{split}
\end{equation}
is a worst-case transition at step $h$.

\subsection{Proof of Lemma 4.1}
\begin{lemma}
\label{variance2}
Under Assumptions \ref{acyclic2} and \ref{Rbound2}, for any worst-case transition $\widehat{p}$ as defined in equation (\ref{p-post2}),
the conditional variance of the robust Q-values under the posterior distribution satisfies the robust Bellman inequality:
\begin{equation*}
    \mathbf{var}_t \widehat{Q}_{sa}^h \leq \nu_{sa}^h + \gamma^2 \sum_{s', a'}\pi_{s'a'}^h \mathbb{E}_t \left(\widehat{p}_{sas'}^h \right) \mathbf{var}_t\widehat{Q}_{s'a'}^{h+1},
\end{equation*}
with $\mathbf{var}_t\widehat{Q}^{H+1} = 0$ and $\nu_{sa}^h := Q^2_{\max} \sum_{s'\in\mathcal{S}}\frac{\mathbf{var}_t \widehat{p}_{sas'} ^h}{ \mathbb{E}_t \widehat{p}_{sas'} ^h}$.
\end{lemma}

\begin{proof}
The proof for the robust setup follows the same line as in \citep{UBE} and is given here for completeness.

First rewrite the conditional variance:
\begin{equation*}
    \begin{split}
      &\text{var}_t(\widehat{Q}_{sa}^h) := \mathbb{E}_t \left( \widehat{Q}_{sa}^h - \mathbb{E}_t \widehat{Q}_{sa}^h \right)^2  \\
      &= \mathbb{E}_t \left( \gamma \inf_{p \in \widehat{\mathcal{P}}_{sa}^h}\sum_{s', a'}\pi_{s'a'}^h p^h_{sas'} \widehat{Q}_{s'a'}^{h+1} - \gamma \mathbb{E}_t \inf_{p \in \widehat{\mathcal{P}}_{sa}^h}\sum_{s', a'}\pi_{s'a'}^h p^h_{sas'} \widehat{Q}_{s'a'}^{h+1} \right)^2\\
      &= \gamma^2 \mathbb{E}_t \left(\sum_{s', a'}\pi_{s'a'}^h \widehat{p}^h_{sas'} \widehat{Q}_{s'a'}^{h+1} -\mathbb{E}_t \sum_{s', a'}\pi_{s'a'}^h \widehat{p}^h_{sas'} \widehat{Q}_{s'a'}^{h+1} \right)^2\\
      &= \gamma^2\mathbb{E}_t \left(\sum_{s', a'}\pi_{s'a'}^h \left(\widehat{p}^h_{sas'} \widehat{Q}_{s'a'}^{h+1} -\mathbb{E}_t \widehat{p}^h_{sas'} \widehat{Q}_{s'a'}^{h+1} \right)\right)^2,
    \end{split}
\end{equation*}
where we used the following definitions:
\begin{equation*}
\begin{split}
  &\widehat{Q}_{sa}^h= r_{sa}^h + \gamma \inf_{p \in \widehat{\mathcal{P}}_{sa}^h}\sum_{s', a'}\pi_{s'a'}^h p_{sas'}^h \widehat{Q}_{s'a'}^{h+1}\\
  &\widehat{p}_{sa}^h \in \arg \inf_{p \in \widehat{\mathcal{P}}_{sa}^h}\sum_{s', a'}\pi_{s'a'}^h p_{sas'}^h \widehat{Q}_{s'a'}^{h+1}.
  \end{split}
\end{equation*}
Assume that $\mathbb{E}_t \widehat{p}_{sas'} >0$ for all $h, s, a, s'$ belonging to the adequate sets. Since any worst-case transition satisfies $\sum_{s'} \widehat{p}^h_{sas'} =1$, we have $\sum_{s', a'}\pi_{s'a'}\mathbb{E}_t \widehat{p}^h_{sas'} =1$ and $\pi_{s'a'}\mathbb{E}_t \widehat{p}^h_{sas'}$ defines a probability distribution over states and actions. Thus, 
\begin{equation*}
\begin{split}
    \quad \mathbb{E}_t \left(\sum_{s', a'}\pi_{s'a'}^h \left(\widehat{p}^h_{sas'} \widehat{Q}_{s'a'}^{h+1} -\mathbb{E}_t \widehat{p}^h_{sas'} \widehat{Q}_{s'a'}^{h+1} \right)\right)^2
    &= \mathbb{E}_t \left(\sum_{s', a'}\pi_{s'a'}^h\frac{\mathbb{E}_t \widehat{p}^h_{sas'}}{\mathbb{E}_t \widehat{p}^h_{sas'}} (\widehat{p}^h_{sas'} \widehat{Q}_{s'a'}^{h+1} -\mathbb{E}_t \sum_{s', a'} \widehat{p}^h_{sas'} \widehat{Q}_{s'a'}^{h+1} )\right)^2\\
    &\leq \sum_{s', a'}\pi_{s'a'}^h \frac{\mathbb{E}_t \widehat{p}^h_{sas'}}{\left(\mathbb{E}_t \widehat{p}^h_{sas'}\right)^2} \mathbb{E}_t\left(\widehat{p}^h_{sas'} \widehat{Q}_{s'a'}^{h+1} -\mathbb{E}_t \sum_{s', a'} \widehat{p}^h_{sas'} \widehat{Q}_{s'a'}^{h+1} \right)^2, 
\end{split}
\end{equation*}
by applying Jensen's inequality to the convex function $x\mapsto x^2$. Therefore, 
\begin{equation*}
\begin{split}
\text{var}_t(\widehat{Q}_{sa}^h) \leq  \sum_{s', a'}\pi_{s'a'}^h \frac{\mathbb{E}_t \widehat{p}^h_{sas'}}{\left(\mathbb{E}_t \widehat{p}^h_{sas'}\right)^2} \mathbb{E}_t\left(\widehat{p}^h_{sas'} \widehat{Q}_{s'a'}^{h+1} -\mathbb{E}_t  \widehat{p}^h_{sas'} \widehat{Q}_{s'a'}^{h+1} \right)^2
\end{split}
\end{equation*}
Rewriting 
$\widehat{Q}_{s'a'}^{h+1} = r_{s'a'}^{h+1} + \gamma \inf_{p \in \widehat{\mathcal{P}}_{s'a'}^{h+1}}\sum_{s'', a''}\pi_{s''a''}^{h+1} p_{s''s'a'}^{h+1} \widehat{Q}_{s''a''}^{h+2}$ 
and 
$\widehat{p}_{sa}^h = \arg \inf_{p \in \widehat{\mathcal{P}}_{sa}^h}\sum_{s', a'}\pi_{s'a'}^h p_{sas'}^h \widehat{Q}_{s'a'}^{h+1}$ 
enables us to claim that under Assumption \ref{acyclic2}, $\widehat{p}_{sa}^h $ is independent of $\widehat{Q}_{sa}^{h+1}$ conditionally on $\mathcal{F}_t$, because $\widehat{Q}_{s'a'}^{h+1}$ depends on downstream uncertainty sets. Note that this claim relies on the rectangular structure of the uncertainty set. Thus,
 
\begin{equation*}
\begin{split}
     \mathbb{E}_t\left(\widehat{p}^h_{sas'} \widehat{Q}_{s'a'}^{h+1} -\mathbb{E}_t  \widehat{p}^h_{sas'} \widehat{Q}_{s'a'}^{h+1} \right)^2 
     &=\mathbb{E}_t\left( (\widehat{p}^h_{sas'} -\mathbb{E}_t \widehat{p}^h_{sas'}) \widehat{Q}_{s'a'}^{h+1} + \mathbb{E}_t  \widehat{p}^h_{sas'} (\widehat{Q}_{s'a'}^{h+1}- \mathbb{E}_t\widehat{Q}_{s'a'}^{h+1}) \right)^2\\
     &=\mathbb{E}_t\left( (\widehat{p}^h_{sas'} -\mathbb{E}_t \widehat{p}^h_{sas'}) \widehat{Q}_{s'a'}^{h+1}\right)^2 + \mathbb{E}_t  \left(\widehat{p}^h_{sas'} (\widehat{Q}_{s'a'}^{h+1}- \mathbb{E}_t\widehat{Q}_{s'a'}^{h+1}) \right)^2.
\end{split}
\end{equation*}
We use the conditional independence property again and Assumption \ref{Rbound2} in order to deduce the following:
\begin{equation*}
\begin{split}
    \mathbb{E}_t\left( (\widehat{p}^h_{sas'} -\mathbb{E}_t \widehat{p}^h_{sas'}) \widehat{Q}_{s'a'}^{h+1}\right)^2 &= \mathbb{E}_t (\widehat{p}^h_{sas'} -\mathbb{E}_t \widehat{p}^h_{sas'})^2 \mathbb{E}_t(\widehat{Q}_{s'a'}^{h+1})^2\leq Q_{\max}^2\mathbf{var}_t\widehat{p}^h_{sas'}, \\
    \text{and }\mathbb{E}_t  \left(\widehat{p}^h_{sas'} (\widehat{Q}_{s'a'}^{h+1}- \mathbb{E}_t\widehat{Q}_{s'a'}^{h+1}) \right)^2 &= \mathbb{E}_t  (\widehat{p}^h_{sas'})^2 \mathbb{E}_t(\widehat{Q}_{s'a'}^{h+1}- \mathbb{E}_t\widehat{Q}_{s'a'}^{h+1}) ^2 = \mathbb{E}_t  (\widehat{p}^h_{sas'})^2 \mathbf{var}_t \widehat{Q}_{s'a'}^{h+1}.
\end{split}
\end{equation*}
Finally, 
\begin{equation*}
\begin{split}
\text{var}_t(\widehat{Q}_{sa}^h) &\leq \gamma^2 \sum_{s', a'}\pi_{s'a'}^h \frac{\mathbb{E}_t \widehat{p}^h_{sas'}}{\left(\mathbb{E}_t \widehat{p}^h_{sas'}\right)^2}(Q_{\max}^2\mathbf{var}_t\widehat{p}^h_{sas'} + \mathbb{E}_t  (\widehat{p}^h_{sas'})^2 \mathbf{var}_t \widehat{Q}_{s'a'}^{h+1}) \\
&\leq \gamma^2 \sum_{s', a'}\pi_{s'a'}^h \frac{\mathbb{E}_t \widehat{p}^h_{sas'}}{\left(\mathbb{E}_t \widehat{p}^h_{sas'}\right)^2} Q_{\max}^2\mathbf{var}_t\widehat{p}^h_{sas'} + \gamma^2 \sum_{s', a'}\pi_{s'a'}^h \mathbb{E}_t \widehat{p}^h_{sas'} \mathbf{var}_t \widehat{Q}_{s'a'}^{h+1}\\
&\leq  Q_{\max}^2 \sum_{s'}\frac{\mathbf{var}_t\widehat{p}^h_{sas'}}{\mathbb{E}_t \widehat{p}^h_{sas'}} + \gamma^2 \sum_{s', a'}\pi_{s'a'}^h \mathbb{E}_t \widehat{p}^h_{sas'} \mathbf{var}_t \widehat{Q}_{s'a'}^{h+1}\\
&\leq \nu_{sa}^h + \gamma^2 \sum_{s', a'}\pi_{s'a'}^h \mathbb{E}_t \widehat{p}^h_{sas'} \mathbf{var}_t \widehat{Q}_{s'a'}^{h+1},
\end{split}
\end{equation*}
where $\nu_{sa}^h$ is given by $\nu_{sa}^h :=  Q_{\max}^2 \sum_{s'}\frac{\mathbf{var}_t\widehat{p}^h_{sas'}}{\mathbb{E}_t \widehat{p}^h_{sas'}}$. 
\end{proof}

\subsection{Proof of Theorem 4.1}

\begin{theorem}[Solution of URBE]
\label{urbe-thm2}
For any worst-case transition $\widehat{p}$ as defined in equation (\ref{p-post2}) and any policy $\pi$, under Assumptions \ref{acyclic2} and \ref{Rbound2}, there exists a unique mapping $w$ that satisfies the uncertainty robust Bellman equation:
\begin{equation}
\label{urbe2}
    w_{sa}^h = \nu_{sa}^h + \gamma^2 \sum_{s'\in\mathcal{S}, a'\in\mathcal{A}} \pi_{s'a'}^h \mathbb{E}_t(\widehat{p}_{sas'}^h)w_{s'a'}^{h+1}, 
\end{equation}
for all $(s,a)\in \mathcal{S}\times\mathcal{A}$ and $h=1, \cdots, H$ where $w^{H+1} = 0$. Furthermore, $w\geq \mathbf{var}_t \widehat{Q}$. 
\end{theorem}

\begin{proof}
Denote by $\mathcal{W}^h$ the robust Bellman operator underlying equation (\ref{urbe2}) and rewrite is as $\mathcal{W}^hw^{h+1} = w^h$. We can easily see that the robust Bellman operator is non-decreasing. 
Also, it has a unique solution, as stated in the following lemma, which is the policy evaluation version of the Min-Max Problem addressed in \citep{bertsekas} (Exercise 1.5).

\begin{lemma}
\label{dp2}
For every $(s,a)\in\mathcal{S}\times\mathcal{A}$, for all step $h=1,\cdots,H$, $w_{sa}^h$ is given by the subsequent steps of the following algorithm which proceeds backwards from $H+1$ to $h$:
 $$
 \left\{
    \begin{array}{ll}
        w_{sa}^{H+1} = 0 \text{ for all } (s,a)\in\mathcal{S}\times\mathcal{A} \\
        w_{sa}^{h} = \nu_{sa}^h +  \gamma^2 \sum_{s'\in\mathcal{S}, a'\in\mathcal{A}} \pi_{s'a'}^h \mathbb{E}_t(\widehat{p}_{sas'}^h)w_{s'a'}^{h+1}
    \end{array}
\right.
$$ 
Therefore, there exists a unique solution to $\mathcal{W}^hw^{h+1} = w^h, h=1,\cdots,H$.
\end{lemma}
 The lower-bound then follows from induction reasoning.  At step $H$, we have $\mathbf{var}_t \widehat{Q}^{H+1} = 0 = w^{H+1}$. Assume that for some $h\leq H$ we have $ w^{h+1} \geq \mathbf{var}_t \widehat{Q}^{h+1}$. Then, by assumption and using Lemma \ref{variance2}, we get:
\begin{equation*}
\begin{split}
    \mathbf{var}_t \widehat{Q}^{h} \leq \mathcal{W}^h\mathbf{var}_t \widehat{Q}^{h+1}
    &\leq \mathcal{W}^h w^{h+1} = w^h.
\end{split}
\end{equation*}
The induction property is hereditary, which concludes the proof of the theorem.
\end{proof}


\section{DQN-URBE Experiments}

\begin{table}[h]
\caption{System's dynamics}
\label{hyperparam}
\begin{center}
\begin{tabular}{lll}
\multicolumn{1}{c}{\bf } & \multicolumn{1}{c}{\bf MARSROVER}  &\multicolumn{1}{c}{\bf CARTPOLE} \\
\hline \\
Nominal model  & $p = 0.005$ & Length = $0.75$, Mass = $1$ \\ \hline
    Size of uncertainty set  & 15 samples & 15 samples\\ \hline

\end{tabular}
\end{center}
\end{table}

\begin{table}[!!h]
\caption{Networks}
\label{hyperparam}
\begin{center}
\begin{tabular}{lll}
\multicolumn{1}{c}{\bf DQN-URBE NETWORKS} & \multicolumn{1}{c}{\bf MARSROVER}  &\multicolumn{1}{c}{\bf CARTPOLE} \\
\hline \\
    Q-network  & ReLu(2 hidden layers of size 10)  & ReLu(3 hidden layers of size 128) \\ \hline
    U(R)BE-network  & ReLu(1 hidden layer of size 15), & ReLu(1 hidden layer of size 100),\\
    & linear activation function for the output& linear activation function for the output \\ \hline
\end{tabular}
\end{center}
\end{table}

\begin{table}[!!h]
\caption{Hyper-parameters}
\label{hyperparam}
\begin{center}
\begin{tabular}{lll}
\multicolumn{1}{c}{\bf DQN-URBE HYPERPARAMETERS} & \multicolumn{1}{c}{\bf MARSROVER}  &\multicolumn{1}{c}{\bf CARTPOLE} \\
\hline \\
    Discount factor $\gamma$ & $0.9$& $0.9$ \\ \hline
    Q-learning rate & 1e-4 & 1e-4 \\ \hline
    U(R)BE network learning rate &1e-4& 1e-4 \\ \hline
    Initial variance coefficient $\mu$ &1e-2& 1e-2\\ \hline
    Posterior parameter $\beta$ &$0.5$ & $0.5$ \\ \hline
    Mini-batch size & $100$ & $256$ \\ \hline
    Final epsilon &1e-3 & 1e-5\\ \hline
    Target update interval & $10$ & $10$ \\ \hline
    Max number of episodes for training  $M_{train}$ & $3000$& $4000$ \\ \hline
    Number of episodes for testing $M_{test}$ & $200$  & $200$  \\ \hline
\end{tabular}
\end{center}
\end{table}

\end{document}